\documentclass[letterpaper, 10 pt, journal, twoside]{IEEEtran}

\makeatletter
\let\proof\@undefined
\let\endproof\@undefined
\makeatother

\usepackage{cite}
\usepackage{url}
\usepackage{amsmath,amssymb,amsfonts}
\usepackage{algorithmic}
\usepackage{graphicx}
\usepackage{textcomp}

\usepackage{enumitem}
\usepackage{multirow}
\usepackage{booktabs}
\usepackage{nicefrac}
\usepackage[caption=false, font=footnotesize]{subfig}
\usepackage{lipsum}

\usepackage[ruled,vlined,linesnumbered]{algorithm2e}
\usepackage{algorithmic,algorithm2e,float}

\SetKw{Continue}{continue}

\newcommand{\vect}[1]{\boldsymbol{#1}}
\newcommand{\be}{\begin{equation}}
\newcommand{\ee}{\end{equation}}

\newcommand{\Pp}{\mathbb{P}}

\newcommand{\E}{\mathbb{E}}
\newcommand{\best}{\mathrm{best}}
\newcommand{\curr}{\mathrm{curr}}
\newcommand{\new}{\mathrm{new}}

\usepackage{tikz}
\usepackage{tikzscale}
\usepackage{filecontents}
\usetikzlibrary{arrows.meta}
\usetikzlibrary{calc}
\usetikzlibrary{shapes.geometric, arrows}
\usetikzlibrary{arrows}
\tikzstyle{startstop} = [rectangle, rounded corners, minimum width=3cm, minimum height=1cm,text centered, draw=black, fill=red!30]
\tikzstyle{io} = [trapezium, trapezium left angle=70, trapezium right angle=110, minimum width=3cm, minimum height=1cm, text centered, draw=black, fill=blue!30]
\tikzstyle{process} = [rectangle, minimum width=3cm, minimum height=1cm, text centered, draw=black, fill=orange!30]
\tikzstyle{decision} = [diamond, minimum width=3cm, minimum height=1cm, text centered, draw=black, fill=green!30]
\tikzstyle{arrow} = [thick,->,>=stealth]
\tikzstyle{decision} = [diamond, draw, fill=blue!20, 
text width=4.5em, text badly centered, node distance=3cm, inner sep=0pt]
\tikzstyle{block} = [rectangle, draw, fill=blue!20, 
text width=5em, text centered, rounded corners, minimum height=2em]
\tikzstyle{line} = [draw, -latex']
\tikzstyle{cloud} = [draw, ellipse, text centered, text badly centered,fill=red!20, node distance=3cm,
minimum height=2em]

\usepackage{amsthm}
\theoremstyle{definition}
\newtheorem{problem}{Problem}

\newtheorem{definition}{Definition}
\newtheorem{proposition}{Proposition}

\newtheorem*{remark*}{Remark}

\newtheorem{lemma}{Lemma}

\def\BibTeX{{\rm B\kern-.05em{\sc i\kern-.025em b}\kern-.08em
		T\kern-.1667em\lower.7ex\hbox{E}\kern-.125emX}}

\urldef{\smith}\url{stephen.smith@uwaterloo.ca}
\urldef{\kulic}\url{dana.kulic@uwaterloo.ca}
\urldef{\wilde}\url{nwilde@uwaterloo.ca}

\title{Bayesian Active Learning for Collaborative\\ Task Specification Using Equivalence Regions}

\author{Nils~Wilde,~\IEEEmembership{Student Member,~IEEE,}
	Dana~Kuli\'{c},~\IEEEmembership{Member,~IEEE,} 
	and~Stephen~L.~Smith,~\IEEEmembership{Senior Member,~IEEE}%
	\thanks{\textcopyright 2019 IEEE.  Personal use of this material is permitted.  Permission from IEEE must be obtained for all other uses, in any current or future media, including reprinting/republishing this material for advertising or promotional purposes, creating new collective works, for resale or redistribution to servers or lists, or reuse of any copyrighted component of this work in other works.}
	\thanks{Manuscript received September, 10, 2018; Revised December, 8, 2018; Accepted January, 13, 2019. This paper was recommended for publication by Editor Dongheui Lee upon evaluation of the Associate Editor and Reviewers' comments. This research is partially supported by the Natural Sciences and Engineering Research Council of Canada (NSERC) and OTTO Motors.}
	\thanks{N.~Wilde and S.~Smith are with the Department of Electrical and Computer Engineering, University of Waterloo, D.~Kuli\'c is with the University of Waterloo and Monash University. (\wilde; \kulic; \smith)}
}

\begin{document}
	
	\maketitle
	
	\begin{abstract}
		Specifying complex task behaviours while ensuring good robot performance may be difficult for untrained users. We study a framework for users to specify rules for acceptable behaviour in a shared environment such as industrial facilities. As non-expert users might have little intuition about how their specification impacts the robot's performance, we design a learning system that interacts with the user to find an optimal solution. Using active preference learning, we iteratively show alternative paths that the robot could take on an interface. From the user feedback ranking the alternatives, we learn about the weights that users place on each part of their specification. We extend the user model from our previous work to a discrete Bayesian
		learning model and introduce a greedy algorithm for proposing alternative that operates on the notion of equivalence regions of user	weights. We prove that with this algorithm the revision active learning process converges on the user-optimal path. In simulations on realistic industrial environments, we demonstrate the  convergence and robustness
		of our approach.
	\end{abstract}
	
	\section{Introduction}
	We address two active research topics in human-robot interaction (HRI): learning from non-expert users and human-robot collaboration. We develop a methodology for non-expert users to create specifications for complex robot tasks \cite{HR_collab_survey}. These specifications then enable humans and robots to operate in a shared workspace \cite{human_robot_proximity}.	
	
	For instance, in an industrial environment shared between humans, autonomous and human-operated vehicles, a facility operator might define road rules to be followed by both robot and human-operated vehicles.  Such road rules increase the predictability of robot behaviour for humans in the environment. These can include constraints such as areas of avoidance, one way roads or speed limits. The environment map, the operator specifications and a defined set of start and goal locations yield a complete specification of a robot task. In practice, designing such rules can be challenging, as operators might have little intuition about how their specification will affect the robot's behaviour and therefore the performance. They might be willing to accept the violation of less important constraints if sufficiently beneficial for task performance. For instance, Figure \ref{fig:example_spec} shows an industrial environment with several user defined constraints. When the robot uses the dark blue path, it respects all constraints. In the alternative solutions, the robot traverses (i.e.,~violates) constraints as this enables a significant reduction in the time to travel from start to goal.\\
	\begin{figure} 
		\centering
		\includegraphics[width=0.45\textwidth]{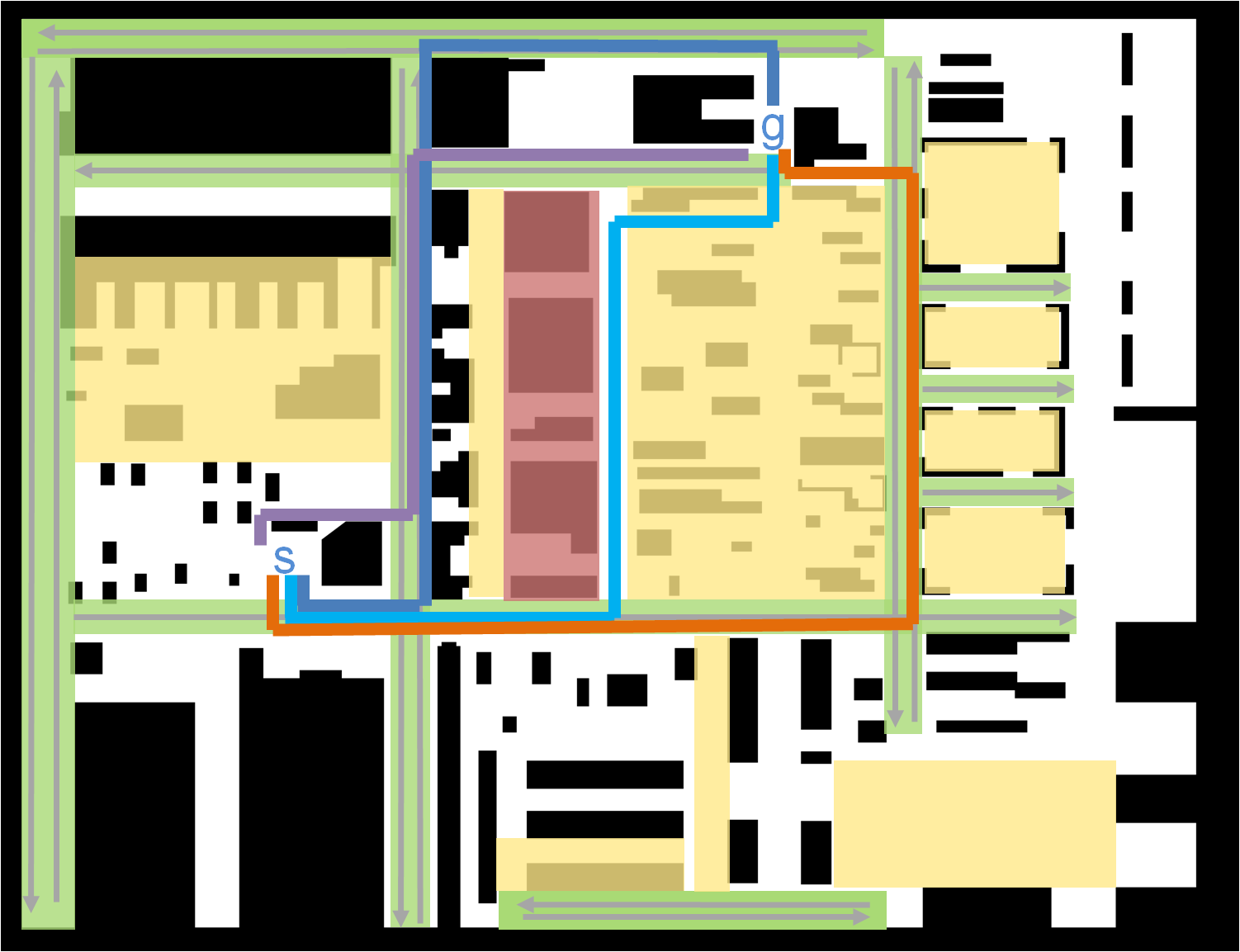}	
		\caption{
			Example environment (white) with obstacles (black) and user defined constraints. Roads are drawn in green with an arrow indicating the direction. Speed limit zones are drawn in yellow, while areas of avoidance are illustrated in red. 
			Further, four different paths between a given start and goal are shown.
			Dark blue indicates the initial path following the specification. Purple and orange paths are alternatives that the simulated user accepted during the interaction. In cyan we show the user optimal path $P^*$, to which the learning eventually converged.			
		}
		\label{fig:example_spec} 
	\end{figure}
	
	A user likely has a preference for which path is a better solution for the given task, based on the completion time and on the importance of the constraints. This can be captured via weights, describing the importance of constraints. However, asking the user to define these weights is unintuitive and possibly challenging. We propose a framework where a user provides only a spatial definition of constraints, while the importance of each constraint is latent.\\ 
	To fill the gap between this incomplete user input and a complete robot task description we apply active preference learning. We present the user with alternative solutions, i.e., paths, and ask for a ranking. In our framework, the user is ''on-the-loop'' and provides feedback at their convenience.	
	The latest  path preferred by the user becomes the current path and can already be executed. Consequently, each set of alternative paths contains the current path.
	Through this interaction with the user, the relative importance of each constraint can be learned. In our previous work \cite{ICRA2018_paper}, we developed an algorithm that iteratively builds a set of linear inequalities on the hidden weights of the user constraints.
	Our user model evaluated paths based on a cost function trading-off constraint violation and time. The learning system assumed that the user would always provide feedback consistent with that cost function and thus iteratively rejected paths that became inconsistent with the user feedback.
	We extend our previous work \cite{ICRA2018_paper} in two ways: First, the assumptions on the user are relaxed in order to capture more realistic behaviour. By considering noisy user feedback and introducing a probabilistic learning approach, we allow users to not always behave consistently with our user model.
	Second, based on our formulation of the problem as a shortest path search on a graph, we define equivalence regions for possible solutions. We propose a new learning system that exploits the notion of equivalence regions, which are sets of constraint weights that are indistinguishable to the user. From this we obtain a greedy algorithm that allows highly efficient learning, outperforming other state-of-the-art techniques.\\
	\setcounter{secnumdepth}{0}
	\paragraph{Related work}
	\setcounter{secnumdepth}{3}
	Recently, active preference learning has been extensively studied for robot task specification. Thereby, a user is assumed to have an internal, hidden cost function which is learned from the user feedback to a presented set of alternatives. For instance, in \cite{activeRewardLearning} experts rank the demonstrated task performance for grasping applications while \cite{dragan_orig, dragan2} focus on continuous trajectories of dynamical systems like autonomous mobile robots.
	We propose a framework relying on a deterministic black-box planner that outputs a path for a given set of weights for the user constraints \cite{ICRA2018_paper}. Different weights can have the same optimal solution, allowing for a discretization of the weight space. 
	Therefore, our problem can be cast as an entity identification problem \cite{decision_tree_problem}: Our hypotheses are the sets of constraint weights that have different optimal solutions, tests correspond to asking the user about their preference between paths and observations equal their feedback. Golovin et al. \cite{EC2_learning_alg} introduce a strong algorithm for near-optimal Bayesian active learning with noisy observations. 
	However, their approach focuses on running each test at most once while we allow for repeated queries.
	While \cite{dragan_orig} greedily reduces the integral of the continuous probability density function over the weight space, our greedy algorithm is formulated over the discretized weight space corresponding to unique paths.
	A major drawback of the user model proposed in \cite{dragan_orig} and \cite{dragan2} is that the user's behaviour depends on the scale of the selected features as it considers an absolute instead of a relative error and does not provide a normalizing mechanism. We propose a more general, scale-invariant user model and show its robustness in simulations.
	Finally, our work differs from other applications of active preference learning for robotics in the way we choose the features for the cost function. Usually, features are picked manually and are therefore a design choice for the learning system \cite{activeRewardLearning, CLAUS}. In our case, features are the violations of constraints that follow from the user specification and are user specific.	
	
	A common technique for learning from demonstration \cite{LfD_trajecotires} is inverse reinforcement learning (IRL). The optimal behaviour of a dynamical system is described by a hidden reward function. IRL then learns this reward function by observing optimal demonstrations. Similarly, we assume a hidden user cost/reward function for the quality of a path. The cost function is modelled as a weighted sum of predefined features and we are interested in learning the weights. However, providing demonstrations might be difficult \cite{activeRewardLearning}, the amount of necessary demonstrations may be prohibitively large \cite{RL_human_pref_pairwise} or demonstrations may require a high level of expertise \cite{learn_traj_prefs}. Providing rich and precise specifications prior to a robot executing a task can be challenging and more prone to inaccuracies \cite{IRL_apprentice_learning}. 
	In contrast, active preference learning learns hidden reward functions by proposing alternative solutions and asking for the user's preference.	
	Closely related to this work, \cite{Alex_GUI} presents a GUI for specifying the user constraints on a given environment, which is used as a front-end to the work presented in this paper.\\
	\paragraph{Contributions}
	In our previous work \cite{ICRA2018_paper}, we proposed a deterministic user model for learning about weights from a ranking feedback and proposed a complete algorithm.
	Using the same framework for combining path planning with user constraints, we extend the user model. To capture user feedback inconsistent with our assumed cost function, we propose a Bayesian learning approach (Section \ref{sec:bayes_learn}). Thereby, we exploit the discrete properties of our problem, introducing a partitioning of the solution space based on equivalence regions. We prove almost sure convergence of the algorithm (Section \ref{sec:prob_alg}) and derive a greedy approach (Section \ref{sec:greedy_policy}). Finally, we show the performance and robustness of our approach in comparison with another state-of-the-art technique in extensive simulations (Section \ref{sec:eval}).
	
	\section{Problem Formulation}
	\label{chap:prob}
	\subsection{Preliminaries}	
	
	Using definitions from \cite{CO_Korte}, a multi-graph is a triple $G=(V, E, \Psi)$, where the function $\Psi: E\rightarrow \{(v,w) \in V \times V: v  \neq w\}$ associates each edge with an ordered pair of vertices. Multiple edges are allowed to connect the same ordered pair of vertices and are then called parallel.
	In our problem we consider doubly weighed multi-graphs of the form $G=(V,E,\Psi, c_1, c_2)$. Thereby, $c_1$ and $c_2$ are independent weight functions, each associating a real number to each edge of the graph: $c_i:E(G)\rightarrow\mathbb{R}$ for $i\in\{1,2\}$.
	
	A walk between two vertices $v_1$ and $v_{k+1}$ on a graph $G$ is a finite sequence of vertices and edges $v_1, e_1, v_2, e_2, \dots , e_k,v_{k+1}$ where $e_1,e_2,\dots e_k$ are distinct.
	A path $P_{v_1,v_{k+1}}$ between two vertices $v_1$ and $v_{k+1}$ is defined as a graph $(\{v_1, v_2,\dots,v_{k+1}\},\{e_1, e_2,\dots,e_{k}\}) $ 
	where $ v_1, e_1, v_2, e_2, \dots , e_k,v_{k+1}$ is a walk. On a weighted graph, the cost of a path is defined as $c(P)=\sum_{e\in P} c(e)$. In doubly weighted graphs we define two costs $c_1$ and $c_2$ where $c_i(P)=\sum_{e\in P} c_i(e)$ for $i\in\{1,2\}$.\\

	\subsection{Problem statement}
	\label{sec:probState}
	%
	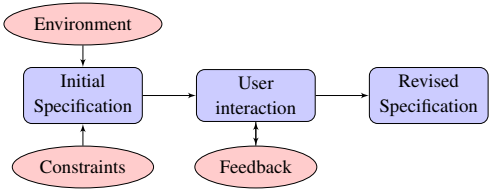
\begin{figure} 
		\centering
		\resizebox {0.75\columnwidth} {!} {
			\begin{tikzpicture}[auto]
			\tikzset{node distance = 2.9cm and 0.2cm}
			\node [block] (init) {\small Initial\\ Specification};
			\node [block, right of=init] (learn) {\small User interaction};
			\node [block, right of=learn] (final) {\small Revised\\ Specification};
			\node [cloud, above of=init, yshift=-1.8cm] (env) {\small Environment};
			\node [cloud, below of=init, yshift=1.8cm] (user1) {\small Constraints};
			\node [cloud, below of=learn, yshift=1.8cm] (user2) {\small Feedback};			
			\path [line] (env) -- (init);
			\path [line] (user1) -- (init);
			\path [line] (init) -- (learn);
			\path [line] (learn) -- (final);
			\path [line][bend right=10] (user2) -- (learn);
			\path [line][bend right=10] (learn) -- (user2);
			\end{tikzpicture}
		}
		\caption{
			Flowchart of the problem. An initial specification is revised using interactive learning to obtain a revision that better fits the user preferences.}
		\label{fig:flowchart} 
	\end{figure}
	
	We summarize the problem setup in Figure \ref{fig:flowchart}. From the environment and a set of user constraints we construct an initial specification for the robot. As users might allow the violation of some of their constraints for sufficient time benefit, we present them with alternative paths during the interaction and ask for feedback. From this feedback we learn the weights of the constraints and obtain a revised specification that corresponds to the user preferences. Figure \ref{fig:example_spec} illustrates how the path that we believe to be optimal evolves after observing user feedback until it converges to the optimal solution.	
	
	As in our previous work \cite{ICRA2018_paper}, we consider a fully known, static environment, represented as a weighted strongly connected multigraph $G'=(V,E,\Psi,t)$. The weight $t$ on the graph encodes the time a robot requires to traverse an edge. We use parallel edges with different times to model speed.
	A robot task consists of navigating from a start vertex $v_{\mathrm{start}}$ to a goal vertex $v_{\mathrm{goal}}$ on $G'$.
	On the environment, a user specifies a set $\Gamma$ containing $d$ constraints. Each constraint is a pair $(E_i, w_i^*)$, where $E_i$ is a subset of the edges of $G'$ and $w_i^*$ is a hidden user cost for the constraint. 
	To incorporate the user specification, we create a doubly weighted graph $G=(V,E, \Psi, t, w^*)$. For each edge $e$ in $G$ the second weight $w^*(e)$ is defined as the sum of all $w_i^*$ that belong to a constraint containing $e$.
	The problem is to find a path from  $v_{\mathrm{start}}$ to $v_{\mathrm{goal}}$ that minimizes the following objective:
	\be
	\min_P\sum_{e\in P}w^*(e)+t(e).
	\label{eq:objective}
	\ee
	The true user weights $w^*_i$ are latent. Moreover, they are defined in units of time, allowing us to pose the multi-objective optimization as an unweighted sum. To learn about the weights, we can query the user by presenting them with a set of paths $\{P^0, P^1,\dots P^k\}$. The feedback is a vector $\vect{u}\in \mathbb{R}^k$ representing a ranking, i.e., a partial ordering, of the presented paths. Without loss of generality we focus only on pair-wise comparisons, as the ranking of additional elements can be expressed with a set of pair-wise relations. This is also well motivated with respect to the user; ranking more than two alternatives might be unnecessarily challenging \cite{pairwise_active_learning}.
	
	We formally define the \emph{Learning of User Preferences} (LUP) problem as follows:
	\begin{problem}[LUP]
		\label{prob:LUP}
		Given a graph $G'$, a user specification $\Gamma$, a start and goal vertex, a user evaluating a presented set of paths and a budget of iterations for querying the user, maximize the belief about the true user weights ${w}_i^*$ and its corresponding shortest path $P^*$ with respect to equation \eqref{eq:objective}.
	\end{problem}
	
	\section{Probabilistic Learning}
	\label{sec:prob_model}
	In this section we propose a probabilistic model of user behaviour. In our previous work \cite{ICRA2018_paper} we required the user to always provide feedback consistent with a linear user model. In contrast, we now consider that the user feedback may be noisy and thus the user feedback is not deterministic.
	
	\subsection{Bayesian Learning}
	\label{sec:bayes_learn}
	Using definitions for Bayesian inference from \cite{statisticsBook} we set up a learning model for gaining information about the hidden (latent) parameter $\vect{w}$. We model the user weights to be positive and finite: $w_i \in [0, w^{\max}]$. 
	The cost of a path $P$ is $C(P)=\vect{\phi}\vect{w}^*+t$, where the violation vector $\vect{\phi}$ describes how many edges of each constraint are traversed by $P$, $\vect{w}^*$ is a column vector containing all latent user weights and $t$ is the time to traverse $P$.
	From each user feedback for a pair $(P^i, P^j)$ we can derive a hyperplane of the form $(\vect{\phi}^i-\vect{\phi}^j)\vect{w}= t^j-t^i$. This hyperplane defines two subsets of the weight space, $\Lambda^{ij}$ and $\Lambda^{ji}$, where $\Lambda^{ij}=\{\vect{w}\in[0,w^{\max}]^d| (\vect{\phi}^i-\vect{\phi}^j)\vect{w}\leq t^j-t^i\}$. Thus $\Lambda^{ij}$ is the set of all weights for which $P^i$ has lower cost than $P^j$.
	\paragraph{Probabilities of halfpspaces}
	For any pair of paths $(P^i, P^j)$, the parameter $\vect w^*\in\Lambda^{ij}$ holds iff $C(P^i)\leq C(P^j)$.
	Adopting a Bayesian perspective, we treat $\vect w^*\in\Lambda^{ij}$ as a random variable and assign an uninformed prior $\Pp(\vect w^*\in\Lambda^{ij})=\nicefrac{1}{2}$. Notice that the volumes of $\Lambda^{ij}$ and $\Lambda^{ji}$ do not correspond to the probability of a path being preferred over another path. 
	From the user feedback about two paths $P^i$ and $P^j$ we obtain binary observations. We denote observations with a random variable $U^{ij}$, indicating whether the user prefers path $P^i$ or $P^j$. A deterministic user always provides feedback where $U^{ij}=1 \iff C(P^i)\leq C(P^j)$, i.e., $U^{ij}=1 \iff \vect{w}^*\in\Lambda^{ij}$. A probabilistic user is consistent with this model with some probability $p^{ij}$. Hence, the probability of $U^{ij}$ given $\vect w^*\in\Lambda^{ij}$ is
	\be
	\begin{aligned}
		&\Pp(U^{ij}=1|\vect w^*\in\Lambda^{ij})=&	
		{p}^{ij} \\
		&\Pp(U^{ij}=1|\vect w^*\notin\Lambda^{ij})=&	
		1-{p}^{ij}.
	\end{aligned}
	\label{eq:observation_model}
	\ee
	We refer to $p^{ij}$ as the accuracy of the user and assume $p^{ij}>\nicefrac{1}{2}$, i.e., that our user model fits the user's decision making better than a random guess. If the parameter $p^{ij}$ is hidden, we can evaluate equation \eqref{eq:observation_model} with an estimate $\hat p$. 
	To simplify notation we write $\Pp(U^{ij} = 1)$ as $\Pp(U^{ij})$.
	In general, $p^{ij}$ is a function of $P^i$ and $P^j$. This allows us to model different levels of the user's accuracy depending on how similar the paths are.
	\paragraph{Probabilities of equivalence regions}
	Equation \eqref{eq:observation_model} describes an observation model for a pair of paths, assigning probabilities to halfspaces. We now assign probabilities to paths instead.
	We observe that not every value in the weight space leads to a unique shortest path, which leads to our definition of equivalence regions. 
	\begin{definition}[Equivalence region]
		If the same path is optimal for two weights ${\vect w}^i$ and ${\vect w}^j$, we call ${\vect w}^i$ and ${\vect w}^j$ \emph{equivalent}. An \emph{equivalence region} of a weight ${\vect w}^i$ is then the set of all  weights that are \emph{equivalent} to the weight: $\Omega(\vect w^i)=\{\vect w^j \in \mathbb{R}^n_{\geq0}|\vect w^j \text{ is \emph{equivalent} to } \vect w^i\}$. 
	\end{definition}	
	We can use equivalence regions to discretize the weight space $[0, w^{\max}]^d$.
	Given a comparison of two paths $(P^i, P^j)$, we introduce a second observation model that describes the probability of user feedback given that the true user weight $\vect{w}^*$ lies in the equivalence region $\Omega'$ of some path $P'$ as
	\be
	\begin{aligned}
		\Pp(U^{ij}|\vect{w}^*\in \Omega')
		=
		\begin{cases}
			p^{ij}, \quad \text{if } \Omega'\subseteq \Lambda^{ij}\\
			1-p^{ij}, \quad \text{if } \Omega'\subseteq \Lambda^{ji}\\
			\nicefrac{1}{2}, \quad \text{otherwise.}
		\end{cases}
	\end{aligned}
	\label{eq:observation_paths}
	\ee
	If an equivalence region lies in both halfspaces $\Lambda^{ij}$ and $\Lambda^{ji}$, we obtain no information from the feedback $U^{ij}$, since not all weights in $\Omega'$ are either feasible or infeasible with the user feedback; expressed in the third case. Let $\mathbb{O}$ be the set of all equivalence regions for a given problem instance. The observation model allows us to express a probability for $\vect{w}^*\in \Omega'$ given an observation $U^{ij}$ as a Bayesian posterior
	\be
	\begin{aligned}
		\Pp(\vect{w}^*\in \Omega'|U^{ij}) 
		= 
		\frac{
			\Pp(U^{ij}|\vect{w}^*\in \Omega')
			\Pp(\vect{w}^*\in \Omega')
		}{
			\underset{\Omega\in \mathbb{O}}{\sum}\;		
			\Pp(U^{ij}|\vect{w}^*\in \Omega)
			\Pp(\vect{w}^*\in \Omega)
		}.
		\label{eq:posterior}
	\end{aligned}
	\ee
	
	Following \cite{statisticsBook}, we write the Bayesian posterior for a series of $n$ observations $U$ for arbitrary pairs of paths as 
	
	\be
	\begin{aligned}
		\Pp(\vect{w}^*\in \Omega'|U) 
		= 
		\frac{
			\underset{U^{ij}\in U}{\prod}
			\Pp(U^{ij}|\vect{w}^*\in \Omega')
			\Pp(\vect{w}^*\in \Omega')
		}{
			\underset{\Omega\in \mathbb{O}}{\sum}\;	
			\underset{U^{ij}\in U}{\prod}
			\Pp(U^{ij}|\vect{w}^*\in \Omega)
			\Pp(\vect{w}^*\in \Omega)
		}.
		\label{eq:post_prob_eq_region}
	\end{aligned}
	\ee
	
	\begin{remark*}
		Notice that this general model does not depend on the exact form of the likelihoods $p^{ij}$, we only require $p^{ij}>\nicefrac{1}{2}$. Therefore, our model could use the likelihood function from \cite{dragan2}. Alternatively, one could fix all $p^{ij}$ to a constant. Then, in contrast to \cite{dragan_orig, dragan2}, the accuracy of the user does not depend on the scaling of the features in the cost function. 
		Moreover, our model increases the robustness towards user feedback that appears inaccurate because the user is considering context that is not described by our features. For instance, in a warehouse an operator might have different preferences for different weekdays or wants a robot to temporarily avoid certain regions. This can not be covered with the current cost function and thus this user would appear erratic to the learner. Finally, when the accuracy is set to one, the deterministic learning model \cite{ICRA2018_paper} is recovered.
		The key advantage of using equivalence regions in equation \eqref{eq:post_prob_eq_region} is that it reduces the complexity of the probability distribution since we now have a discrete distribution over regions rather than a continuous one 
		This allows for a significantly faster solving of the problem, as we will show in Section \ref{sec:eval}.
	\end{remark*}
	
	\subsection{Probabilistic Algorithm}
	\label{sec:prob_alg}
	In Algorithm \ref{Alg:generalAlg_prob} we propose a general procedure to iteratively learn about user preferences from pairwise user feedback with inaccurate users.
	Initially, we compute the set of all equivalence regions $\mathbb{O}$ (line 2). After updating our current belief about all equivalence regions (line 4),  we iteratively generate new paths (line 5) similar to our deterministic algorithm from \cite{ICRA2018_paper}. Then, we request user feedback for the pair $(P^{\curr},P^{\new})$ and add the user feedback to a set (6-7). After adding the new observation to our set, we update the weight space and, if necessary, the current weight (8-9). The procedure is repeated until we reach the iteration budget $N$ in line 2, at the end we return the weight $\vect{\hat{w}}^{\best}$ where the posterior belief is maximized.
	We discuss an implementation of the function $\mathtt{getNewPath}(\cdot)$ in Section \ref{sec:greedy_policy}.
	
	\begin{algorithm}[h]	
		\DontPrintSemicolon 
		\KwIn{$G'$, $\Gamma$, $N$}
		\KwOut{$\hat{\vect w}^{\curr}$}			
		
		$\hat{\vect w}^{\curr}=\vect w^{\max}$, $U=\emptyset$\\ Calculate $\mathbb{O}$\\
		\For{$n=1$ to $N$} {
			Update $\Pp(\vect{w}^*\in \Omega'|U) $ for all $\Omega'\in\mathbb{O}$\\
			
			$P(\hat{\vect w}^{\new})
			\leftarrow 
			\mathtt{getNewPath}(
			\mathbb{O},
			\hat{\vect w}^{\curr},\quad\quad\quad\quad\quad$    
			$\{\Pp(\vect{w}^*\in \Omega^1|U),\Pp(\vect{w}^*\in\Omega^2|U),\dots\}\;
			)$\\
			
			Get user feedback $U^{\curr,\new}$ for paths $P(\hat{\vect w}^{\curr})$ and $P(\hat{\vect w}^{\new})$\\	
			
			$U=U\cup U^{\curr,\new}$\\	
			
			\If{$U^{\curr,\new}=\new$}
			{$\hat{\vect w}^{\curr}=\hat{\vect w}^{\new}$}
			
		}	
		\Return{
			$\vect{\hat{w}}^{\best} = \underset{\vect{{w}}'| \Omega' \in \mathbb{O}}{\arg \max}\; \Pp(\vect{w}^*\in \Omega'|U)$
		}\;
		
		\caption{Learning user weights by sampling}
		\label{Alg:generalAlg_prob}
	\end{algorithm}
	
	\paragraph{Convergence}
	We now establish almost surely convergence of Algorithm \ref{Alg:generalAlg_prob}. Let $\vect{w}^*$ be the true user weight and $p^{ij}>0.5$ for all pairs of paths $(P^i,P^j)$. Without loss of generality, we only consider $i,j$ pairs that are ordered such that $C(P^i)\leq C(P^j)$; hence $\vect{w}^*\in\Lambda^{ij}$ always hold (but this is not known to the algorithm). 
	Moreover, $l$ is the number of equivalence regions in $\mathbb{O}$ and $m$ the number of all pair-wise comparisons. 
	For the following definition we change our notation and denote the optimal path $P^*$ as $P^1$.
	\begin{definition}[Asymptotically completely informative sequence]
		\label{def:informative}	
		Let $X_n$ be a sequence of pairs of paths presented to the user in $n$ iterations, and for each $j$, let $X^{1j} = \{X^{1r_1}_1 , X^{1r_2}_2, X_{n_j}^{1 r_{n_j} }\}$ be the longest subsequence of	$X$ for which $\Omega^j \subseteq \Lambda^{1r_1} , \Omega^j \subseteq \Lambda^{1r_2} , . . . , \Omega^j \subseteq \Lambda^{1 r_{n_j}}$.  Then the sequence is asymptotically completely informative if as $n$ goes to $\infty$, we have $n_{j} \to \infty$ for all $j$.	
	\end{definition}
	In other words, if a sequence of pairs of paths contains observations about every $(P^*,P^j)$, and the number of observations for each $(P^*,P^j)$ pair goes to infinity as the length of the sequence goes to infinity, it is called \emph{asymptotically completely informative}.
	Notice that such a sequence of paths is not required to contain subsequences for all pairs of paths $(P^*, P^j)$; it is sufficient if the feedback to the pairs $(P^*, P^j)$ contains information about all other equivalence regions according to \eqref{eq:observation_paths}.
	Finally, we call $U$ \emph{asymptotically completely informative} if the corresponding sequence of paths is \emph{asymptotically completely informative} and treat the probability that the true user weight $\vect{w}^*$ or an estimate $\vect{\hat{w}}$ lie in an equivalence region $\Omega$ as a random variable.	
	
	\begin{proposition}[Convergence]
		\label{prop:prob_convergence}
		Let $\Omega^*$ be the equivalence region containing $\vect{w}^*$. Given \emph{asymptotically completely informative} user feedback $U$, the probability that the best estimate $\vect{\hat{w}}^{\best}$ of Algorithm 1 lies in $ \Omega^*$ converges almost surely to 1 as all $n_j$ go to infinity:
		\be
		\Pp(\vect{\hat{w}}^{\best}\in \Omega^*|U)=1,\; n_j\to\infty,\; 
		\text{for all }j.
		\ee	
	\end{proposition}
	
	\begin{proof}
		At first consider the comparison of an arbitrary pair $(P^i,P^j)$ and fix $P^i=P^*$. Let $U$ be a sequence of user feedback of length $n^{ij}$and $k^{ij}$ be the number of times the user chooses $P^i$, i.e., chooses accurately. Moreover, let $\hat{p}^{ij}$ be our estimate of $p^{ij}$. For simplification we drop the ${ij}$ superscript. 
		From $p>\nicefrac{1}{2}$, we can conclude that $k>\nicefrac{n}{2}$ as $n\to\infty$, using Hoeffding's inequality \cite{hoeffding}. 
		We notice that the probability for a sequence of user feedback given $\vect w^*\in\Lambda^{ij}$ depends on $p$, while our belief about $\vect w^*\in\Lambda^{ij}$ given some user feedback is based on $\hat{p}$. 
		Given user feedback $U$ with known $n$ and $k$, the posterior probability is	
		\be
		\begin{aligned}
			\Pp	(\vect w^*\in\Lambda^{ij}|U)
			=&
			\frac{{\hat{p}^k (1-\hat{p})^{n-k}}}
			{{\hat{p}^k (1-\hat{p})^{n-k}}+{\hat{p}^{n-k} (1-\hat{p})^k}}\\
			=&\frac{1}
			{1+\frac{\hat{p}}{1-\hat{p}}^{n-2k}}.		
		\end{aligned}
		\ee	
		We take the limit as $n$ goes to $\infty$:
		\be
		\begin{aligned}
			\lim\limits_{n\to\infty}
			\Pp(\vect w^*\in\Lambda^{ij}|{U})=
			1=
			\lim\limits_{n\to\infty}
			\frac{1}
			{1+\frac{\hat{p}}{1-\hat{p}}^{n-2k}}.	
			\\		
		\end{aligned}
		\label{eq:proof_limit}
		\ee
		
		Using $\hat{p}>\nicefrac{1}{2}$ leads to $\frac{\hat{p}}{1-\hat{p}}>1$. The term ${n-2k}$ is strictly negative if $k>\nicefrac{n}{2}$. Hence, $\nicefrac{\hat{p}}{(1-\hat{p})}^{n-2k}$ approaches zero as $n$ goes to infinity. We conclude that $\lim\limits_{n\to\infty}	\Pp(\vect w^*\in\Lambda^{ij}|{U})=1$.
		As we only have two paths, we only have two equivalence regions. Following our ordering of $i$ and $j$, $\Omega^*=\Lambda^{ij}$. Hence, $\Pp(\vect{\hat{w}}\in \Omega^*|{U})=1,$ as $ n^{ij}\to\infty$ for a single, fixed pair of paths $P^i$ and $P^j$.\\	
		Finally, we extend the result to comparisons for multiple pairs. Equation \eqref{eq:post_prob_eq_region} expresses the probability of $\vect{w}^*$ lying in a given equivalence region.
		Notice that ${\bigcap}_{j\neq i}\Lambda^{ij}\subseteq\Omega^*$, as well as $\lim\limits_{n\to\infty}\Pp({\vect w^*\in\Lambda^{ij}}|{U})=1$.	
		As $\Pp(\vect{w}^{\best}\in\Lambda^{ij}|{U})\to1$ for all $j\neq i$, we have $\Pp(\vect{w}^{\best}\in{\bigcap}_{j\neq i}\Lambda^{ij}|U)\to1$. 
		Hence, $\Pp(\vect{\hat{w}}^{\best}\in \Omega^*|{U})\to 1$  and the statement holds.
	\end{proof}
	
	From Proposition \ref{prop:prob_convergence} we conclude that Algorithm \ref{Alg:generalAlg_prob} always elicits the true user weight if an \emph{asymptotically completely informative} sequence of paths is presented to the user for feedback, and if the user's accuracy with respect to our model is greater than $\nicefrac{1}{2}$. However this does not include any guarantees on the speed of convergence. In the next section we derive a greedy approach to maximize convergence speed.

	\subsection{Greedy Policy}
	\label{sec:greedy_policy}
	We now show how to find new paths in each iteration of Algorithm \ref{Alg:generalAlg_prob}, i.e., the function $\mathtt{getNewPath}$. 
	We notice that computing the set of all equivalence regions for a given problem instance is computationally intractable, a proof is provided in Appendix \ref{sec:appendix_hardness}.
	Thus, the set can be of exponential size in relation to the number of constraints. Because of this, we propose a greedy algorithm for finding new paths.
	We define $q(\vect{w}^*\in \Omega|U)$ as the unnormalized posterior, i.e., the numerator of equation \eqref{eq:post_prob_eq_region}. As $q$ is not a probability we refer to it as the \emph{posterior measure}.
	The decrease in the posterior measure is captured as
	\be
	f(X_n)=1-\sum_{\Omega^i\in \mathbb{O}}q(\vect{w}^*\in \Omega^i|U_n).
	\label{eq:submod_fn_prob}
	\ee
	
	Our primary motivating application is one in which the user is ''on-the-loop''. We do not require them to constantly provide feedback and already execute the current solution $P^{\curr}$ \cite{ICRA2018_paper}. Therefore, we keep the current best path and fix $P^{\curr}$ to be one of the two alternative paths comprising the next query (Algorithm \ref{Alg:generalAlg_prob}).
	Thus, our greedy algorithm returns the path maximizing the posterior measure
	\be
	P^{\new}_n = \underset{P^j}{\arg\max}
	\quad
	\E 
	\left[
	f\left(
	\,(P^{\curr},P^j)\bigcup X_{n-1}\,
	\right)
	\right]
	.
	\label{eq:greedy_obj}
	\ee
	
	In this optimization we only need to consider one $P^j$ for each equivalence region.
	In Appendix \ref{sec:appendix_submod} we
	discuss the case where two new paths are presented, i.e., we
	do not fix one path in each query to be $P^{\curr}$. In this case it	can be shown that equation \eqref{eq:greedy_obj} is an adaptive submodular function.
	Finally, we ensure convergence for our greedy approach.
	\begin{lemma}[Convergence of the greedy algorithm]
		The greedy algorithm equation \eqref{eq:greedy_obj} returns an \emph{asymptotically completely informative} sequence of paths if the number of iterations goes to infinity and thus the probability of the true user weight converges to one almost surely.
	\end{lemma}
	\begin{proof}
		To prove the statement we show two properties: 1) The greedy algorithm eventually returns $P^*$ and 2) if $P^{\curr}=P^*$ it eventually returns all paths necessary to constitute an \emph{asymptotically completely informative} sequence.
		To show the first statement let $P^{\curr}\neq P^*$. 
		If a path $P^j\neq P^*$ is returned we either do not learn about $P^*$ and the posterior of either $\Omega^{\curr}$ or $\Omega^j$ decreases relatively to the posterior of $\Omega^*$ (see equation \eqref{eq:observation_paths}). Otherwise, the comparison of $(P^{\curr},P^j)$ contains information about $\Omega^*$ and thus is expected to increase the posterior of $\Omega^*$.
		While $P^{\curr}\neq P^*$, the expected marginal reward of presenting $P^*$, i.e., $\E 	\left[	f\left(	\,(P^{\curr},P^*)\bigcup X_{n-1}\,
		\right)
		\right] -f(X_{n-1})$, increases monotonically, relatively to the reward of any other path. Thus, $P^*$ will eventually be the maximizer of equation \eqref{eq:greedy_obj}. Then, the greedy algorithm returns $P^*$ and the user will prefer $P^*$ over the current $P^{\curr}$ in expectation.
		
		For the second statement, assume we already have $P^{\curr}=P^*$. Due to inaccurate user feedback another path $P^i$ becomes $P^{\curr}$. However, as shown above, the algorithm eventually presents $P^*$ again.	
		Consider the path $P^j$ where $n_j$ is minimal among all paths  ($n_j$ is defined in Definition \ref{def:informative}). Case 1: $q(\vect{w}^*\in \Omega^j|U)$ is the maximizer of equation \eqref{eq:greedy_obj}, thus $P^j$ is presented and $n_j$ increments. Case 2: Some $q(\vect{w}^*\in \Omega^r|U)$ is the maximizer and $(P^{\curr}, P^r)$ is presented. If the corresponding feedback contains information about $P^j$, $n_j$ increments as well. According to equation \eqref{eq:observation_paths}, if no information about $P^j$ is obtained, $q(\vect{w}^*\in \Omega^j|U)$ increases by a ratio of $\nicefrac{0.5}{(1-p)}$ (where $p$ is the user accuracy) relative to all $q(\vect{w}^*\in \Omega'|U)$, where $\Omega'$ gets rejected by the feedback to $(P^{\curr}, P^r)$. As this holds for all other paths and the number of paths is finite, $q(\vect{w}^*\in \Omega^j|U)$ increases relative to all other posterior measures until, after a finite number of iterations, either information about $P^j$ is obtained and $n_j$ increments, or case 1 applies. Hence, we are guaranteed to increment the minimal $n_j$ and thus all $n_j$ must go to infinity.
	\end{proof}
	
	Performing an exact greedy step is hard, as finding the set of all equivalence regions $\mathbb{O}$ is intractable. In practice, a polynomial sized estimate $\mathbb{O}'$ can be found via sampling which allows for an approximate greedy step.

	\section{Evaluation}
	\label{sec:eval}
	To generate realistic simulations, we recruited users to create specifications. Given the layout of a real industrial facility, users defined constraints as described in \cite{Alex_GUI}. To systematically evaluate our approach, we simulate user feedback in the active learning. This allows us to pick different ground truths $\vect{w}^*$ and generate the user feedback with varying accuracy levels. Figure \ref{fig:example_spec} illustrates an example specification with different possible solutions.
	Further, for an outdoor scenario we conducted experiments using graphs generated by a probabilistic method \cite{Planning_LaValle} and random specifications.	

	Our primary interest is how the posterior belief about $\vect{w}^*$ evolves. In the evaluation we have two objectives: Showing the robustness of our user model and comparing our work with \cite{dragan_orig}. We refer to their approach as the \emph{Maximum Volume Removal (MRV)}  and name our approach the \emph{Maximum Equivalence Region Removal (MERR)} \footnote{Note: In both approaches neither volume nor equivalence regions are removed, we rather assign a lower posterior probability to the rejected items.}.
	
	In our implementation of MRV, we modify the query selection: As we consider a discrete space, queries are found by iterating over $\mathbb{O}'$ rather then solving a continuous optimization problem. To ensure comparability with our framework, we fix one path in the pair that comprises a query as the current path. Moreover, MRV requires a scaling of the features. The user's accuracy is modelled as an exponential function of the difference in the cost of two paths \cite{dragan_orig}. Thus the user's accuracy depends on the scaling of the cost function.
	The model is extended by a linear parameter $\beta$ in \cite{dragan2} to describe different levels of accuracy. However, as no restriction on the scale of the features is made, these $\beta$ values do not yield similar results for different scenarios. In our experiments we manually determined $\beta$ for each scenario such that the user's accuracy is approximately $0.9$.		
	
	To investigate the performance we used three different specifications that vary in complexity. The first specification consists of 26 constraints, covering $33\%$ of the free space, the second (shown in Figure \ref{fig:example_spec}) has 41 constraints covering $40\%$ while the third consists of 52 constraints covering $73\%$\footnote{When a user specifies a road on the interface it counts as 2 constraints for the planner: A reward for following the road, i.e., a constraint with a negative weight, and a penalty for going against the direction of travel.}.
	For each specification we varied between two start and goal pairs and three randomly selected true user weights $\vect{w}^*$. As finding the set of all equivalence regions $\mathbb{O}$ is intractable, we generate estimates $\mathbb{O}'$ via sampling. 
	The graph $G'$ for these experiments is based on a grid layout and has of $5185$ vertices.
	MERR and MVR differ in three components: The model for the simulated user feedback, the user model assumed by the learning system and the strategy for presenting new paths, i.e., the query selection. 
	
	\subsection{Performance of MERR and MVR query selections}
	\paragraph{Experiment 1} 
	
	In the first experiment, we compare the performance for the different active query selections - either MVR or MERR - for a user following the MVR model.
	Figure \ref{fig:Exp1} shows the result for a total of 180 trials (10 repetitions for each configuration of user, start-goal pair and true user weight) with a budget of 30 iterations.
	
	\begin{figure} 
		\centering		
		\includegraphics[width=0.485\textwidth]{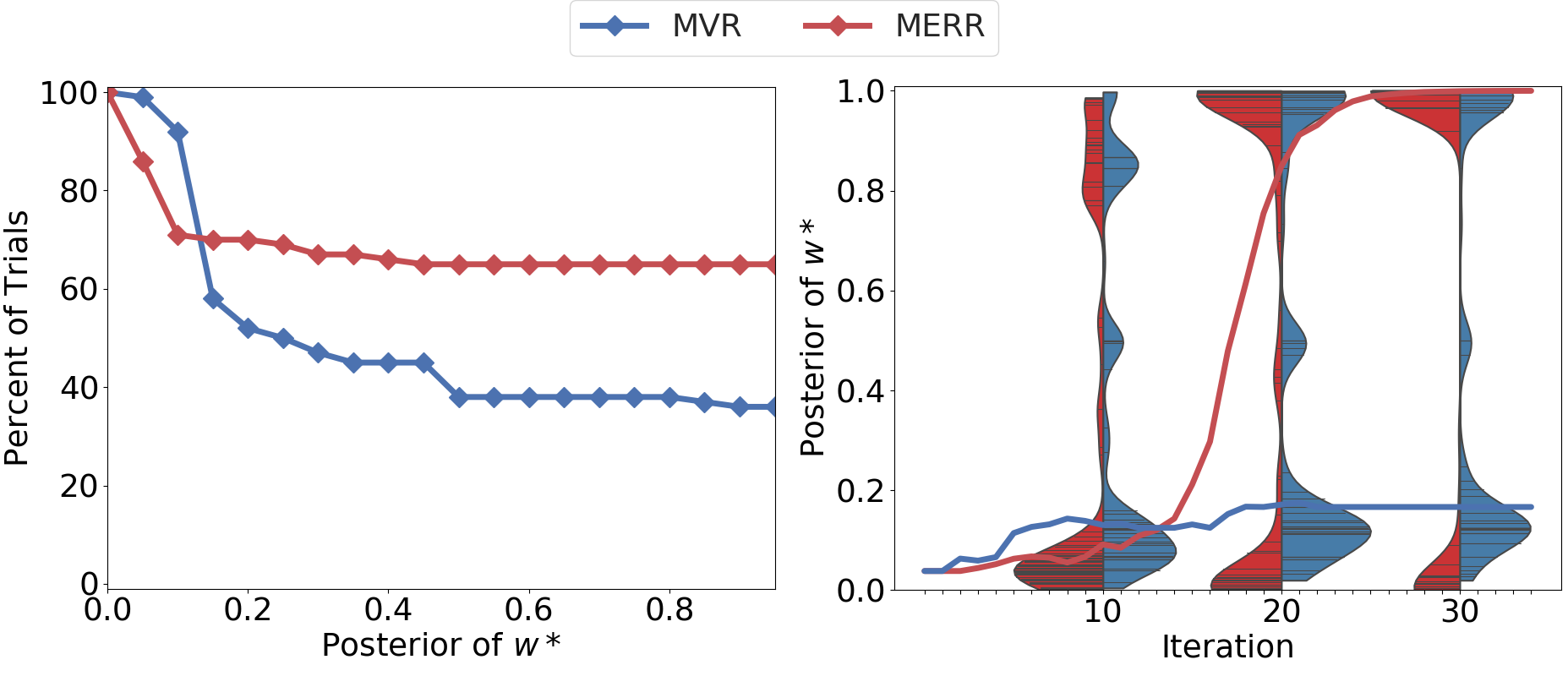}
		\caption{
			Results for experiment 1. For different values of the posterior of the optimal weight $\vect{w}^*$ the left plot shows the percentage of trials that achieved that value within $30$ iterations. In the right plot we show median values of the posterior over all $30$ iterations together with violin plots of the distribution of the posterior of $\vect{w}^*$ at iterations 10, 20 and 30.
		}
		\label{fig:Exp1} 
	\end{figure}
	The left plot of Figure \ref{fig:Exp1} illustrates the percentage of trials that achieved a given final posterior value for the true user weight $w^*$ within the iteration budget. A critical threshold for the posterior is $0.5$, 
	as then $\vect{\hat{w}}^{\best}=\vect{w}^*$ and $\vect{\hat{w}}^{\best}$ is the unique maximizer of the posterior distribution.
	The MERR query selection has a higher success rate: $70\%$ of the trials converge within $30$ iterations, while only $40\%$ do so for MVR queries. Interestingly, both approaches always converge once the posterior of $\vect{w}^*$ surpasses $0.5$.
	In the right sub plot we illustrate the evolution of the median posterior of $\vect{w}^*$ over the iterations. Further, at iterations 10, 20 and 30 we show the distribution of the data.
	We observe that between iterations 10 and 20 the MERR posterior median starts to increase quickly, passing the $0.5$ threshold at iteration 17 and getting past $0.9$ after 21 iterations. MVR shows a slower increase and does not pass $0.2$ within the 30 iterations. The violin plots at three stages of the process illustrate a further detail: The distributions are nearly bimodal. Both approaches succeed for some instances very quickly while the posterior stays low for harder instances. However, we observe that the low end of the distribution shrinks more quickly for MERR and eventually becomes completely bimodal. In some hard instances, the algorithm takes longer to initially show $P^*$ to the user and does not learn about $\vect{w}^*$ until then. However, once $P^*$ is shown, the belief about $\vect{w}^*$ is maximized quickly, leading to the bimodal distribution.	
	
	To explain the better performance of MERR, we recall that equivalence regions vary drastically in volume. MVR often proposes queries that reduce the integral of the posterior but do not significantly change the posterior of equivalence regions and thus makes little progress.
	
	\begin{figure} 
		\centering
		\includegraphics[width=0.485\textwidth]{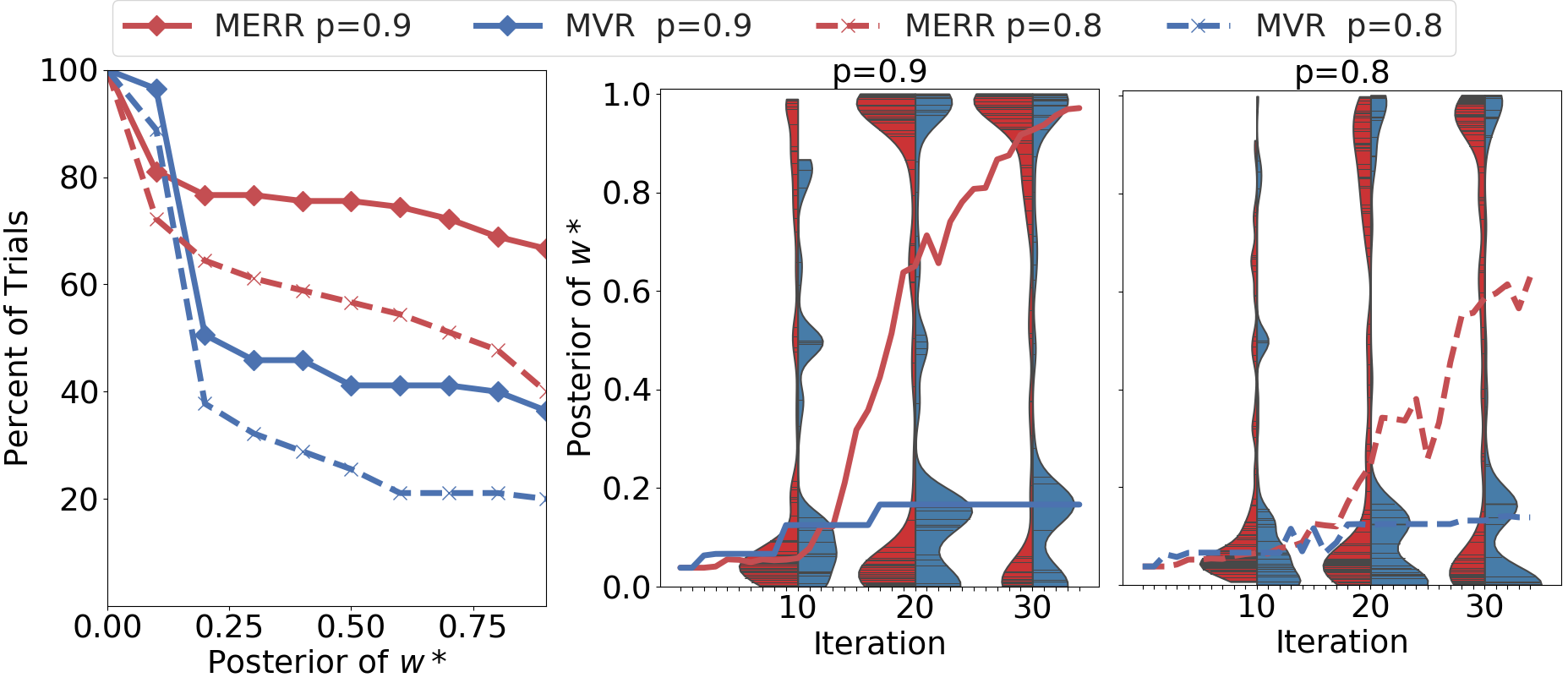}
		\caption{
			Results for experiment 2. The same analysis as in Figure \ref{fig:Exp1} is shown, but for the MERR user two different values for accuracy, $p=0.9$ and $p=0.8$, are depicted.}
		\label{fig:Exp2} 
	\end{figure}
	
	\paragraph{Experiment 2}The second experiment focuses on the performance of both query selection methods, assuming the user generates responses according to the MERR model, i.e., equation \eqref{eq:observation_paths}. In Experiment 1 we simulated the user according to \cite{dragan_orig, dragan2}. In that model, the accuracy depends on how different the presented paths are and therefore is influenced by the query selection. In order to ensure comparability with our user model, we fixed the accuracy to the average of Experiment 1, thus $p=0.9$. Additionally, a second dataset shows the results for lower accuracies of $p=0.8$. Notice that the range for meaningful values is $(0.5, 1]$; users with $p=0.5$ act completely independently of our model. Further, $p=0.9$ corresponds to $10\%$ misleading feedback, $p=0.8$ doubles the error rate to $20\%$ of cases. In Figure \ref{fig:Exp2} we summarize the data as done for Experiment 1.
	
	In contrast to Experiment 1, both query selection methods achieve a lower convergence rate for $p=0.9$. 
	MERR reaches a posterior median of $0.6$ in $80\%$ of the trials, while with MVR less than $45\%$ of trials reach the $0.5$ threshold.
	For the lower accuracy of $p=0.8$, both query selection methods perform worse: MERR converges only in $42\%$ and MVR $20\%$ of cases.
	The graphs illustrating the mean posteriors over the iterations show a similar result for $p=0.9$ compared to Experiment 1. For $p=0.8$, the median of the posterior for MERR makes substantially better progress than MVR. The distributions confirm this observation: MERR shrinks the bottom lobe and gains on the upper end more quickly.
	
	In summary, the first two experiments highlight the performance benefit when maximizing the decrease in the posterior summed over equivalence regions compared to the decrease in the integral of the posterior (i.e., the removed volume), irrespective of the user model. 
	
	\subsection{Robustness of MERR}
	We further investigate how sensitive the proposed approach is to knowledge of the user's accuracy. We simulate the user according to the MERR model with a constant accuracy $p$, but the learning system only has access to an estimate $\hat{p}$.
	We fixed $p=0.7$ and $p=0.85$ and picked either $\hat{p}=p$, or $\hat{p} = p\pm0.1$. The experiment is based on the smallest and largest specifications with 26 and 52 constraints, respectively. For each $p,\hat{p}$ configuration we average over 20 trials.
	
	\begin{figure} 
		\centering		
		\includegraphics[width=0.49\textwidth]{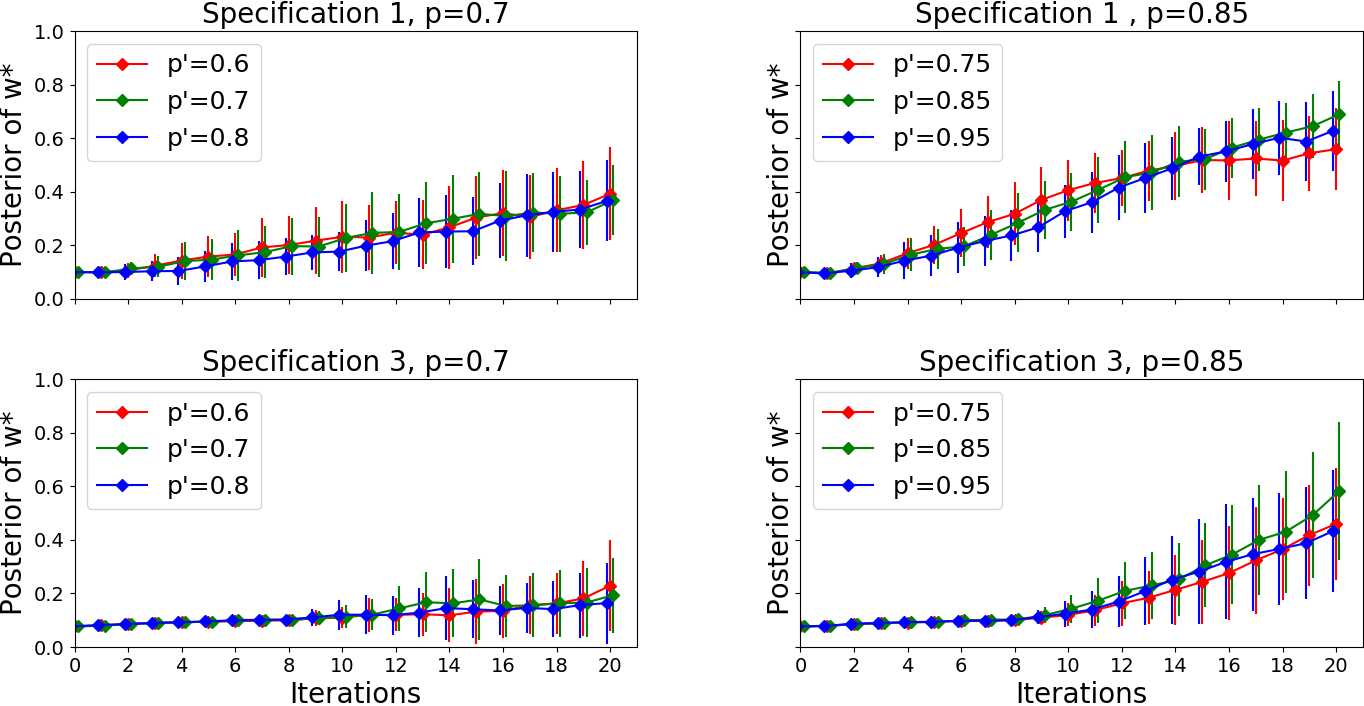}
		\caption{Experiment 3. Robustness of the MERR user model for two specification (26 and 52 constraints) and different accuracies $p$ in the simulated user and different estimates  $p'$ for the learning.}
		\label{fig:Exp3} 
	\end{figure}
	
	Figure \ref{fig:Exp3} shows	that an estimation error of $0.1$ for the user accuracy has very little influence on the performance. In all 4 plots the over- and underestimates behave similarly to when the learner knows the user's accuracy exactly. As expected, the accuracy itself has an impact on the performance: For $p=0.7$ the learning system performs worse for both specifications. Especially for the large specification, there is only little progress over the 20 iterations. In a more complicated setting additional feedback is needed to elicit the true user weight, the higher amount of inaccurate feedback then has a larger impact.
	On the other hand, for $p=0.85$ the final result is relatively similar for both specifications. We conclude that a richer specification has a smaller impact on the performance when the user feedback is more accurate.
	
	\subsection{Extension to other scenarios}
	Finally, we applied the approach to a different setting: An environment described by a graph based on a $k$-nearest probabilistic roadmap (PRM) \cite{Planning_LaValle}. User specifications are generated randomly by sampling polygons in the environment; for each region a constraint is formulated over all edges incident with a vertex in the region. Using the layout of the campus of the University of Waterloo, we generated a PRM graph with $2000$ vertices and $k=10$, the number of sampled constraints is $50$. In Figure \ref{fig:Exp4} we show the map together with the generated graph. Further, we compare the initial path of the learning which does not violate any constraints, and the user optimal path $P^*$ that is learned through interaction.
	We conducted a similar analysis as in Experiments 1 and 2, averaged for 20 different specifications. Overall $50\%$ of the trials achieve a posterior of at least $0.9$ within 50 iterations. Moreover, after $15$ iterations the median passes the $0.5$ threshold, while $0.9$ is reached after $40$ interactions.

	\begin{figure} 
		\centering	
		\subfloat[]{%
			\includegraphics[width=0.2\textwidth]{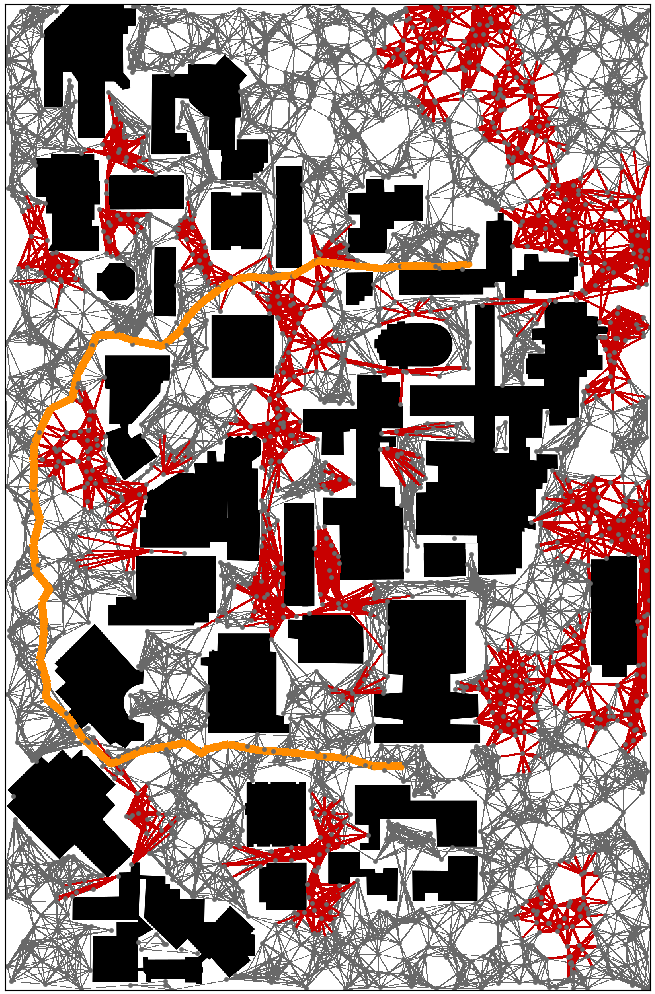}}\hfill
		\subfloat[]{%
			\includegraphics[width=0.2\textwidth]{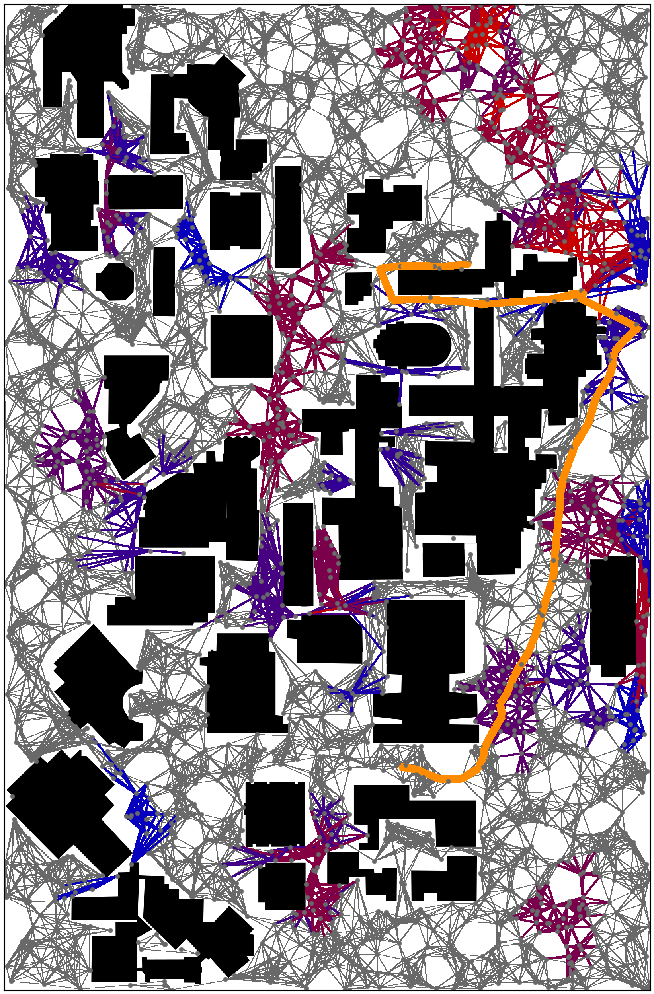}}
		\caption{
			Results for Experiment 4. (a) shows the outdoor environment (buildings in black and freespace in white) with the generated PRM-graph. Red indicates edges that belong to a randomly generated constraint, orange shows the optimal path between a start and goal location of a task when not violating any constraints. (b) shows the optimal path $P^*$ and the updated weights on the constraints -- red indicates high and blue low weights.}
		\label{fig:Exp4} 
	\end{figure}
	\section{Discussion}
	\label{sec:discussion}
	\paragraph{User models}
	Generally, both models, MVR and MERR, assume that the user evaluates paths based on a weighted sum of features. In MVR the user's accuracy depends on how similar the presented paths are. This approach is well motivated and promises good performance when the features are adequate. On the other hand, it has disadvantages with the scaling of features and lacks robustness when users do not follow the model. In contrast, our approach generally models inaccuracies as a random noise and is agnostic towards their exact form. In the simulations we fixed $p$ to a constant and demonstrated the robustness. 
	Therefore, our learning system is less dependent on our user model exactly capturing the real user behaviour.
	A limitation of the MERR user model with constant accuracy values is that accurate user feedback is potentially not exploited efficiently as the noise then is query-independent. This approach is investigated in \cite{CLAUS}. Moreover, both learning models depend on sampling to perform the greedy step. Even though the accuracy can be increased arbitrarily with more samples, finding the optimal solution is computationally intractable.
	\paragraph{Performance in experiments}
	In the first two experiments we showed data that was collected for different user weights and different user specifications (and thus features). Both have a direct influence on the algorithm's performance. Usually, more complex specifications have a larger set of equivalence regions, which affects the convergence. Maybe surprisingly, the true user preference $\vect{w}^*$ can also influence the performance, especially in the MERR model: The learning system makes little progress if most of the hyperplanes learned from a sequence of user feedback intersect $\Omega^*$.
	Moreover, in single trials both models can perform relatively poorly by random chance. An inaccurate user feedback for a query containing $P^*$ leads to decrease in the posterior of $\vect{w}^*$. Then, the learning system might need multiple iterations to present $P^*$ again and thus elicit $\vect{w}^*$. This effect is enhanced when the accuracy of the user is low. Nonetheless, the MERR learning model is still guaranteed to converge.

	\section{Conclusions and Future Work}
	We presented an interactive framework for robot task specification. Based on Bayesian active leaning we derived a greedy algorithm for generating queries.
	Our approach exploits the fact that different weights for
	constraints do not necessarily lead to different optimal paths. Using equivalence regions allows for a discrete Bayesian
	learning model that does not require the user to always
	provide feedback consistent with the assumed cost function. The probability of an inconsistent user feedback is of a general form and scale-invariant. In simulations, we demonstrated that our approach outperforms a related state-of-the-art technique \cite{dragan_orig} and showed robustness of our user model.
	One future work direction is to extend the concept of equivalence regions to continuous spaces by introducing a notion of path similarity.
	Further, the proposed framework can be applied to more complex and realistic scenarios including multiple start and goal locations and additional features, potentially including dynamic data such as traffic.
	Finally, user studies are required to show the practical performance of the framework. 
	
	\appendix
	\subsection{Hardness of finding all non-equivalent paths}
	\label{sec:appendix_hardness}
	
	\begin{proposition}[Hardness of finding all paths]
		\label{prop:allPaths}
		Finding the number of non-\emph{equivalent} paths between $v_{\mathrm{start}}$ and $v_{\mathrm{goal}}$ on the graph $G$ is $\#$P hard.
	\end{proposition}
	\begin{proof}
		To prove the statement we reduce the S-T-paths problem to our problem. The S-T paths problem finds the number of all paths from a start to a goal vertex on an unweighted graph and is known to be \#P complete \cite{complexitiy_sharpP}. Let $(G^{ST}, S, T)$ be an instance of S-T-paths where $G^{ST}=(V^{ST},E^{ST},\Psi^{ST})$ is a non-weighted graph without parallel edges and $\Psi^{ST}$ maps a pair of vertices to each edge. We construct an instance of our problem where $G'=(V,E,\Psi)$ with $V=V^{ST}$, $E=E^{ST}$, $\Psi=\Psi^{ST}$ and $v_{\mathrm{start}}=S$ and $v_{\mathrm{goal}}=T$. We choose $t_i$ for all $e_i\in E$ such that $G'$ is metric.
		We then pick a user specification consisting of $|E|$ user constraints, each defining a weight for exactly one edge such that each edge $e_i$ in $E$ is associated with a single hidden weight $w^*_i$. We obtain a doubly weighted graph $G=(V,E,\Psi, t, w^*)$.
		Every path on $G^{ST}$ then has a corresponding path on  $G$ and vice versa.
		Moreover, any path $P$ between $v_{\mathrm{start}}$ and $v_{\mathrm{goal}}$ on $G$ is a shortest path for some realization of all $w_i$ as we can choose 
		$w^*_i=0 $ if $ e_i\in P$ and $ w^*_j=\sum_{k}t_k$ for all $e_j\notin P$ and $k=1,\dots, |E|$ otherwise. Hence, the number of \emph{equivalence} regions of our problem equals the number of all paths on $G$, which corresponds to the number of S-T paths on $G^{ST}$. We conclude that a solution to finding all equivalent paths solves S-T paths.
	\end{proof}
	The complexity class of $\#$P includes problems such as counting the number of solutions of NP-hard problems; however, even for many problems solvable in polynomial time, counting the solutions is nonetheless $\#$P hard \cite{complexitiy_sharpP}. If a polynomial time algorithm for solving a $\#$P hard problem exists, it would imply P $=$ NP.
	
	\subsection{Discussion of adaptive submodularity}
	\label{sec:appendix_submod}
	In Section \ref{sec:greedy_policy} we proposed a greedy algorithm that maximizes the reduction of an unnormalized posterior. The objective of the algorithm is related to adaptive submodular functions, introduced in \cite{adaptive_submod}.
	A similar approach is presented in the active learning framework of \cite{dragan_orig}, where the objective is the reduction of the unnormalized integrated posterior of the weight space, referred to as the removed volume.
	The authors show that this volume removal function is adaptive submodular.
	In contrast, equation (9) sums over the posterior measure of all equivalence regions. This indicates how the belief over paths changes instead of over all weights.
	When $P^i$ is not fixed to be $P^{\curr}$, our greedy objective function $f(X_n)$ can also be shown to be normalized, adaptive monotone, and adaptive submodular. Adaptive monotonicity follows from $q(\vect{w}^*\in \Omega|U)$ being multiplied with $p^{ij}$, $1-p^{ij}$ or $\nicefrac{1}{2}$ when user feedback is observed. Further, adaptive submodularity follows since the marginal reward of an element $X^{ij}$, as defined in \cite{adaptive_submod}, is smaller for a set $X_m$ than for a set $X_n$, where $|X_m|\geq|X_n$ as the decrease in the posterior measure has an upper bound of  $q(\vect{w}^*\in \Omega^i|U_{n},U^{ij})$. Adaptive submularity provides strong performance guarantees for a greedy approach: At any given iteration, the greedy solution achieves $1-\nicefrac{1}{e}\approx 0.63$ times the optimal solution and is the best polynomial time approximation \cite{adaptive_submod}.	
			
	\bibliographystyle{IEEEconf}

\end{document}